\documentclass[sigconf]{acmart}
\AtBeginDocument{%
  }

\usepackage{xspace}
\usepackage{microtype}
\usepackage{subfigure}
\usepackage{hyperref}
\usepackage{algorithmic,amsmath,xcolor} 
\usepackage{color,soul}

\usepackage{mathtools}
\usepackage{amsthm}
\usepackage[capitalize,noabbrev]{cleveref}
\usepackage{multirow}
\usepackage{xcolor}
\usepackage{makecell}

\usepackage{enumitem}
\setlist[itemize]{leftmargin=*}

\usepackage[textsize=tiny]{todonotes}

\usepackage{algorithm}

\theoremstyle{plain}
\newtheorem{theorem}{Theorem}[section]
\newtheorem{proposition}[theorem]{Proposition}

\theoremstyle{definition}

\theoremstyle{remark}

\AtBeginDocument{%
  }

\copyrightyear{2024}
\acmYear{2024}
\setcopyright{acmlicensed}\acmConference[CIKM '24]{Proceedings of the 33rd ACM International Conference on Information and Knowledge Management}{October 21--25, 2024}{Boise, ID, USA}
\acmBooktitle{Proceedings of the 33rd ACM International Conference on Information and Knowledge Management (CIKM '24), October 21--25, 2024, Boise, ID, USA}
\acmDOI{10.1145/3627673.3680017}
\acmISBN{979-8-4007-0436-9/24/10}
 
\begin{document}

\newcommand{\sysname}{\textit{Blind-Match}\xspace}
\newcommand{\bt}{Blind-Touch\xspace} 

\title{\sysname: Efficient Homomorphic Encryption-Based 1:N Matching for Privacy-Preserving Biometric Identification}

\author{Hyunmin Choi}
\orcid{0009-0002-0486-9582}
\affiliation{%
  \institution{NAVER Cloud}
  \city{Seongnam}
  \country{Republic of Korea}
}
\email{hyunmin.choi@navercorp.com}

\author{Jiwon Kim}
\orcid{0009-0006-5905-1883}
\affiliation{%
  \institution{Sungkyunkwan University}
  \city{Suwon}
  \country{Republic of Korea}
}
\email{merwl0@g.skku.edu}

\author{Chiyoung Song}
\orcid{0009-0006-1938-1917}
\affiliation{%
  \institution{NAVER Cloud}
  \city{Seongnam}
  \country{Republic of Korea}
}
\email{chiyoung.song@navercorp.com}

\author{Simon S. Woo}
\authornote{Corresponding authors.}
\orcid{0000-0002-8983-1542}
\affiliation{%
  \institution{Sungkyunkwan University}
  \city{Suwon}
  \country{Republic of Korea}
}
\email{swoo@g.skku.edu}

\author{Hyoungshick Kim}
\orcid{0000-0002-1605-3866}
\authornotemark[1]
\affiliation{%
  \institution{Sungkyunkwan University}
  \city{Suwon}
  \country{Republic of Korea}
}
\email{hyoung@skku.edu}
    
\renewcommand{\shortauthors}{Hyunmin Choi, Jiwon Kim, Chiyoung Song, Simon S. Woo, \& Hyoungshick Kim}


\begin{abstract}
We present \sysname, a novel biometric identification system that leverages homomorphic encryption (HE) for efficient and privacy-preserving 1:N matching. \sysname introduces a HE-optimized cosine similarity computation method, where the key idea is to divide the feature vector into smaller parts for processing rather than computing the entire vector at once. By optimizing the number of these parts, \sysname minimizes execution time while ensuring data privacy through HE. \sysname achieves superior performance compared to state-of-the-art methods across various biometric datasets. On the LFW face dataset, \sysname attains a 99.63\% Rank-1 accuracy with a 128-dimensional feature vector, demonstrating its robustness in face recognition tasks. For fingerprint identification, \sysname achieves a remarkable 99.55\% Rank-1 accuracy on the PolyU dataset, even with a compact 16-dimensional feature vector, significantly outperforming the state-of-the-art method, Blind-Touch, which achieves only 59.17\%. Furthermore, \sysname showcases practical efficiency in large-scale biometric identification scenarios, such as Naver Cloud's FaceSign, by processing 6,144 biometric samples in 0.74 seconds using a 128-dimensional feature vector.
\end{abstract}

\bibliographystyle{ACM-Reference-Format}

\maketitle
       
\section{Introduction}
\label{sec:introduction}


Biometric identification, such as fingerprint, facial, and iris recognition, is commonly used for user authentication on personal devices~\cite{cho2020security}. However, its adoption in web and cloud environments is limited due to the difficulty in changing or revoking biometric data once compromised, as demonstrated by the 2016 US Office of Personnel Management breach, where 5.6 million individuals' fingerprints were stolen~\cite{gootman2016opm}. This incident highlights the critical need for secure management of biometric data on servers.

 Biometric template protection (BTP) techniques, such as locality-sensitive hashing (LSH), have gained attention for securely storing biometric data~\cite{sandhya2017biometric, meden2021privacy, manisha2020cancelable, hahn2022biometric}. LSH projects biometric templates into a hash space, allowing efficient matching while obfuscating the original template. However, recent studies~\cite{paik2023security,dong2019genetic,ghammam2020cryptanalysis,lai2021efficient} have demonstrated vulnerabilities in LSH-based BTP techniques, as hash codes still contain significant information about the original template. Attackers can exploit the similarity-preserving properties of hash codes to reverse-engineer the biometric data, compromising the privacy and security of the authentication system.





In contrast to LSH-based techniques, homomorphic encryption (HE) offers a promising solution for secure biometric identification in server and cloud environments. HE enables computations on encrypted data without decryption~\cite{gentry2009fully, cheon2017homomorphic}, allowing biometric data to remain \emph{fully encrypted} throughout the identification process on the server side. This provides strong privacy protection and mitigates the risks associated with data breaches and unauthorized access. Numerous researchers have developed HE-based privacy-preserving fingerprint identification~\cite{engelsma2019learning, yang2020secure, Kim20:bio} and face identification~\cite{boddeti2018secure, ibarrondo2023grote} techniques. However, the computational overhead introduced by HE operations often leads to slow data matching or searching times, making real-world deployment challenging. Consequently, there is a pressing need for efficient HE-based biometric identification techniques that can achieve practical performance.





\begin{figure*}[!ht] 
\centerline{\includegraphics[width=1.91\columnwidth]{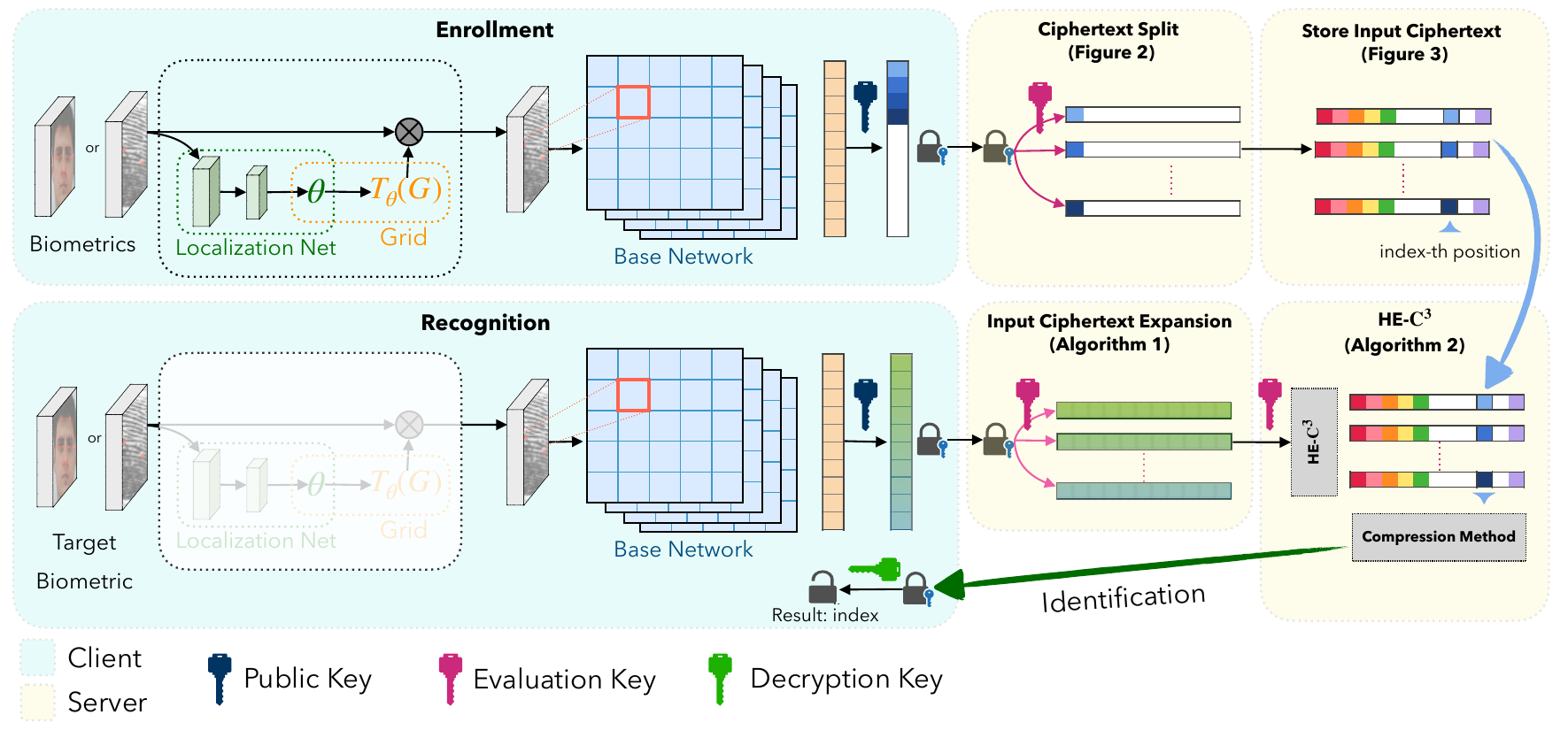}}
\caption{Overview of \sysname. \sysname consists of two stages: During the enrollment phase, \sysname divides and stores the encrypted feature vector into smaller parts. During the recognition phase, \sysname processes these smaller parts through multiplication and then aggregates the results using our new cosine similarity computation method.}
\label{fig:look_and_feel}
\end{figure*}

Choi et al.~\cite{choi2023blindtouch} recently introduced Blind-Touch, a highly efficient HE-based fingerprint authentication system. However, Blind-Touch employs a fully connected (FC) layer-based metric function, differing from the cosine similarity-based methods typically used in standard biometric authentication architectures, such as SphereFace~\cite{sphereface2017}, CosFace~\cite{cosface2018}, and ArcFace~\cite{arcface2019}. This FC-layer-based metric usage in Blind-Touch results in a critical limitation in the accuracy of the 1:N (one-to-many) matching task, which is inherently more challenging than the 1:1 matching task.

The 1:N matching task is crucial in real-world scenarios, such as Naver Cloud's FaceSign~\cite{facesign}, which enables users to make payments using facial recognition without a credit card. In this service, the system must identify or authenticate an individual by searching through a large database of biometric data without additional identity information. For practical implementation, FaceSign requires the matching task to handle over 5,000 face images within 1 second. This demands a highly accurate and efficient 1:N matching algorithm to ensure a seamless user experience and maintain security. However, our experimental results show that Blind-Touch~\cite{choi2023blindtouch} significantly underperforms in the 1:N matching task, achieving only 59.17\% Rank-1 accuracy on the PolyU dataset~\cite{lin2018matching}, despite its sufficient speed. This limitation makes the Blind-Touch architecture unsuitable for services like FaceSign, which require both high accuracy and fast processing speed for the 1:N matching task.

To address the challenge of privacy-preserving 1:N biometric matching, we introduce a novel method optimized for HE that computes cosine similarity. The key idea is to partition the feature vector into smaller parts for processing rather than handling the entire vector at once. This technique has led to the development of \sysname, a robust biometric identification system that achieves Rank-1 accuracies of 99.68\% on the PolyU dataset and 99.63\% on the LFW dataset, respectively, outperforming the best HE-based method, Blind-Touch~\cite{choi2023blindtouch}. We optimized \sysname's computational efficiency by identifying the ideal number of parts to process independently and employing power-of-2 values for template splitting. These techniques allow \sysname to enable fast 1:N biometric matching, making it suitable for real-world scenarios. The overall architecture of \sysname is illustrated in Figure~\ref{fig:look_and_feel}. Our key contributions are summarized as follows:

\begin{itemize}
    \item \textbf{Designing a New HE-Optimized Cosine Similarity Algorithm:} We introduce a novel cosine similarity computation method specifically designed to optimize performance in the context of HE. The key idea behind this approach is to divide a biometric feature vector into multiple smaller parts and process each part individually. This technique allows the system to efficiently process large-scale biometric databases, demonstrating its practicality for real-world applications. In our experiments, \sysname exhibits its impressive speed by processing 6,144 biometric samples in just 0.74 seconds, highlighting its ability to perform a matching task with over 5,000 face images within 1 second, meeting the established criteria for real-world deployment.
    
    
    \item \textbf{Demonstrating \sysname's Superiority:} The experimental results show that \sysname significantly outperforms Blind-Touch~\cite{choi2023blindtouch} in fingerprint recognition tasks. On the PolyU fingerprint dataset, \sysname achieves a 99.55\% Rank-1 accuracy, which is 40.38\% higher than Blind-Touch. Moreover, across various face datasets, \sysname consistently achieves high Rank-1 accuracy exceeding 94\%, demonstrating its practicality and robustness for real-world biometric identification systems.

    \item \textbf{Releasing Source Code and Preprocessed Datasets:} We make \sysname's source code and preprocessed fingerprint dataset publicly available for reproducibility on the GitHub site: \url{https://github.com/hm-choi/blind-match}. We believe this can further improve the research in this area and promote and demonstrate practical implementation of our approach in real-world settings.
\end{itemize}
\section{Related Work}
\label{sec:relatedwork}

\label{subsec:homomorphic_encryption}
\textbf{Homomorphic Encryption (HE): } Homomorphic Encryption (HE) enables computations on encrypted data without decryption, allowing data owners to entrust third parties with statistical or machine learning operations while maintaining data privacy. Following Gentry's pioneering work in 2009~\cite{gentry2009fully}, various HE schemes have been developed, such as the BGV scheme for integer-based arithmetic operations~\cite{brakerski2014leveled} and the CKKS scheme for approximate arithmetic operations over encrypted real and complex numbers~\cite{cheon2017homomorphic}. The CKKS scheme has been widely used to preserve privacy in machine learning applications.

The CKKS scheme processes data in batches, with the batch size known as the \textit{number of slots}. The maximum number of allowable multiplications is called the \textit{depth}, which is fixed during key pair generation. If the number of multiplications exceeds the \textit{depth}, the accuracy of the decryption result is not guaranteed. The CKKS scheme allows three operations: Addition ($Add$), Multiplication ($Mul$), and Rotation ($Rot$).

Let $N$ be the \textit{number of slots} and $\mathbf{v_1}, \mathbf{v_2}$ be two real vectors of size $N$. $C(\mathbf{v_1})$ and $C(\mathbf{v_2})$ denote the ciphertexts of $\mathbf{v_1}$ and $\mathbf{v_2}$ respectively. The above operations are more formally defined as follows:

\begin{itemize}[noitemsep]
\item Add (P): $Add(C(\mathbf{v_1}), \mathbf{v_2}) = C(\mathbf{v_1} \oplus \mathbf{v_2})$
\item Add (C): $Add(C(\mathbf{v_1}), C(\mathbf{v_2})) = C(\mathbf{v_1} \oplus \mathbf{v_2})$
\item Mul (P): $Mul(C(\mathbf{v_1}), \mathbf{v_2}) = C(\mathbf{v_1} \otimes \mathbf{v_2})$
\item Mul (C): $Mul(C(\mathbf{v_1}), C(\mathbf{v_2})) = C(\mathbf{v_1} \otimes \mathbf{v_2})$
\item Rot : $Rot(C(\mathbf{v}), r) = C(v_{r}, v_{r+1}, ... , v_{N-1}, v_{1}, ... , v_{r-1})$, \\
where $\mathbf{v} = (v_{1}, v_{2}, ..., v_{N})$ and $r$ is a non-zero integer.
\end{itemize}

Here, (P) denotes operations between a ciphertext and a plaintext (or a constant), while (C) denotes operations between two ciphertexts. The notation $\oplus$, and $\otimes$ represents element-wise addition, and multiplication. In this paper, we only use $Add(C)$ for addition, denoted as $Add$. The $Mul$ notation represents multiplication with re-linearization, which reduces the ciphertext elements from three to two after multiplication. The $Res$ notation represents the rescale operation that reduces noise after multiplication while maintaining ciphertext precision. However, the level $l$ of the ciphertext decreases to $l-1$ after the $Res$ operation, indicating the remaining number of allowed multiplications.

In this work, we chose the CKKS scheme for \sysname due to its efficient support for real number operations, making it the most suitable option for our privacy-preserving machine learning application. In particular, we adopt the \textit{conjugate invariant ring} setting suggested by Kim et al.~\cite{kim2019approximate}, which supports only real numbers.

\subsection{Biometric Recognition with HE}
In recent years, there has been a growing interest in developing biometric recognition systems that leverage HE to protect user privacy. There have been multiple efforts to implement fingerprint and face recognition, while preserving privacy using HE.

\noindent\textbf{Face Matching Task.} Boddeti et al.~\cite{boddeti2018secure} propose a fast HE-based 1:1 face matching algorithm, but their architecture is limited to 1:1 matching. As a result, matching time increases proportionally with the number of registered users, leading to poor performance in 1:N matching scenarios (over 10 seconds for a feature vector size of 128 with 1,000 users). To address the 1:N matching task, recent models like SphereFace~\cite{Liu_2017}, CosFace~\cite{cosface2018}, ArcFace~\cite{arcface2019}, AdaCos~\cite{adacos2019}, and CurricularFace~\cite{curricular2020} have effectively employed cosine similarity. Ibarrondo et al.~\cite{ibarrondo2023grote} proposed a group testing HE-based face identification using cosine similarity. However, their study used a naive computation method rather than an optimized cosine similarity method for HE, resulting in slow processing speeds.

\noindent\textbf{Fingerprint matching task.} Kim et al.~\cite{kim2020efficient} developed a HE-based 1:1 matching algorithm utilizing the TFHE library~\cite{TFHE}. However, this method requires a considerable amount of time, approximately 166 seconds, for a single matching. Moreover, DeepPrint~\cite{engelsma2019learning} introduced HE-based 1:N matching, but it requires over 3.4 seconds to search among 5,000 fingerprints, not including the time for encryption and network communication. Additionally, the size of the input ciphertext is about 62 MB. Recently, Blind-Touch~\cite{choi2023blindtouch} has been proposed as a promising solution for real-time HE-based fingerprint 1:1 matching with high accuracy. However, our experiments have demonstrated a significant degradation in Blind-Touch's performance for the 1:N matching task.

To overcome the limitations of existing methods, we introduce \sysname, which supports standard biometric architectures such as SphereFace, ArcFace, and CosFace by using HE-based cosine similarity computation for high performance in the 1:N matching task for both face and fingerprint recognition scenarios.

\section{Overview of \sysname}
\label{sec:overview}

The primary objective of \sysname is to optimize the computation of cosine similarity within the HE context, enabling support for standard biometric identification architectures. By dividing the feature vector into smaller ciphertexts and efficiently processing them on the server side, \sysname significantly reduces the number of computationally expensive operations, such as rotations, while ensuring accurate matching results.

\subsection{Key Generation and Management}
\label{subsec:key_gen}

In \sysname, a key set consisting of a public key, evaluation key, and decryption key is generated. The public key is distributed to each client device for encrypting extracted feature vectors from biometric data. The evaluation key is used for $Mul$ and $Rot$ operations on the server. The decryption key is securely stored on the client device, while only the public key and evaluation key are sent to the server. Modern devices like smartphones and laptops facilitate secure storage through integrated mechanisms, ensuring the privacy and security of the decryption key.


\subsection{Enrollment}
\label{subsec:enrollment}

During enrollment, a client encrypts a user $u$'s feature vector and transmits the resulting ciphertext $C_{u}$ to the server. The feature vector occupies the initial slots of the ciphertext, with the remaining slots padded with zeros. The \textit{number of slots} in the ciphertext is determined by the degree parameter $N$. Figure~\ref{fig:input_ctxt_structure}(a) shows the structure of $C_u$.

\begin{figure}[!ht]
\centerline{\includegraphics[width=0.92\columnwidth]{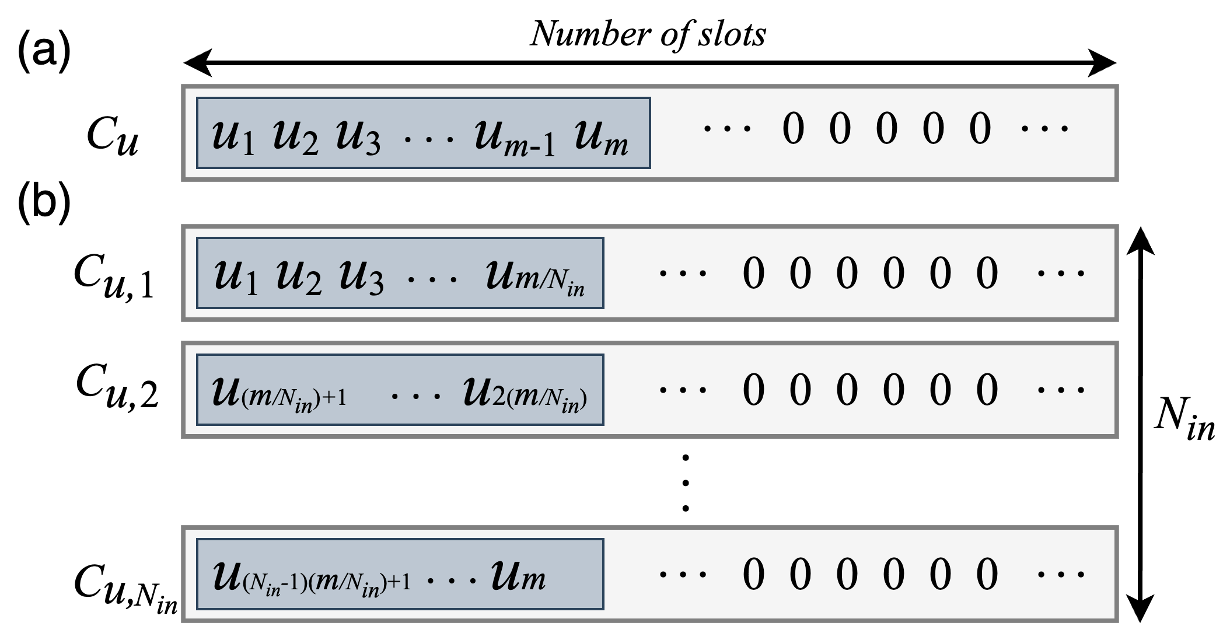}}
\caption{Structure of input ciphertexts.}
\label{fig:input_ctxt_structure}
\end{figure}

The server multiplies the received $C_u$ by a \textbf{mask}, a set of $N_{in}$ sub-vectors $\textbf{mask}_{i}$, each of length $N$. The slots of $\textbf{mask}_{i}$ are filled with zeros, except for $m/N_{in}$ slots starting from the $i \times m/N_{in}$th index, which are filled with ones, where $m$ represents the size of the feature vector extracted from the CNN extractor (see Figure~\ref{fig:input_ctxt_structure}(b)). Without loss of generality, the degree parameter $N$ and feature vector size $m$ are defined as powers of two.
Each multiplied ciphertext is then rotated right by $(\textit{Index}-i+1) \times m/N_{in}$ slots, where $\textit{Index}$ represents the user's unique identifier. These rotated ciphertexts are added to the previously stored ciphertexts, positioning the user's feature vector at the $\textit{Index} \times m/N_{in}$ position in the combined ciphertext. The server stores these combined ciphertexts, as described in Figure~\ref{fig:rotate_and_add}.

\begin{figure}[t]
\centerline{\includegraphics[width=0.92\columnwidth]{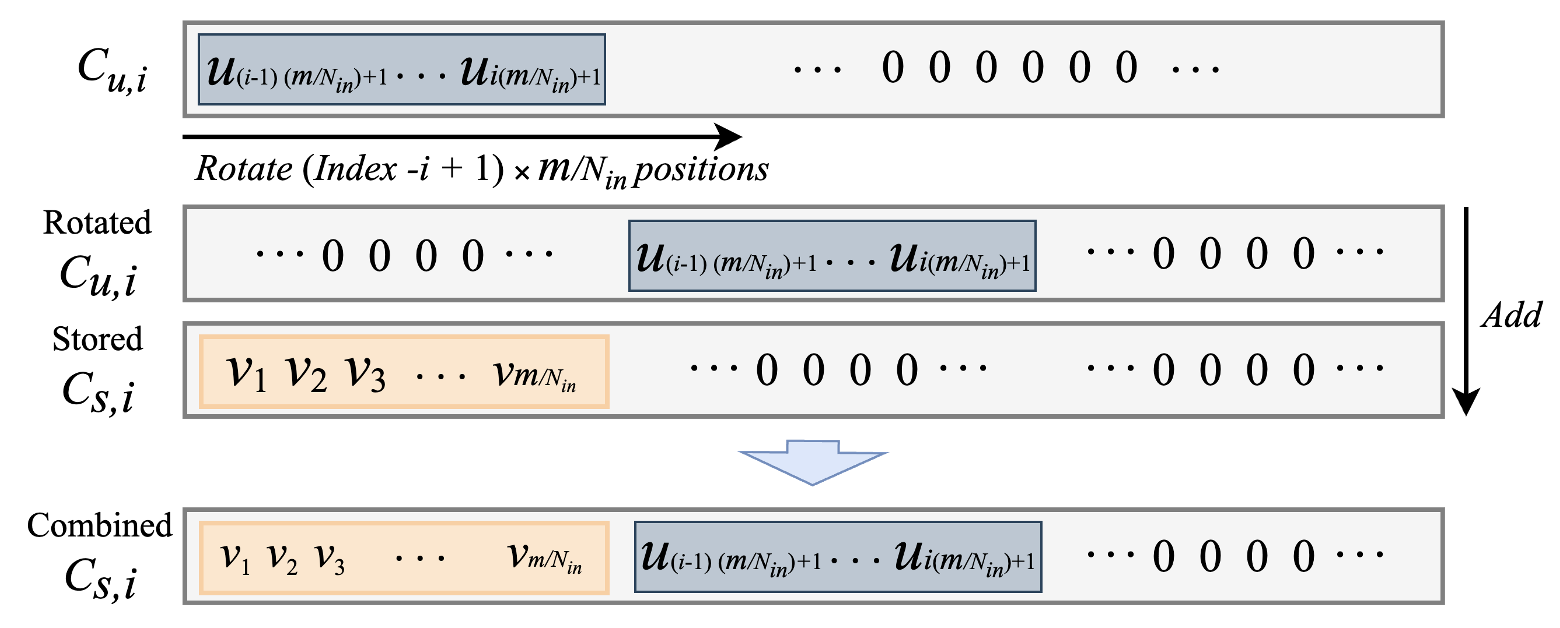}}
\caption{Enrollment of the $i$-th input ciphertext.}
\label{fig:rotate_and_add}
\end{figure}
 
\subsection{Recognition}
\label{subsec:recognition}
During recognition, a client encrypts the user's feature vector and transmits the resulting ciphertext $C_u$ to the server. The server then expands $C_u$ into $N_{in}$ ciphertexts using the Input Ciphertext Expansion algorithm (Algorithm~\ref{alg:input_ctxt_expansion_methods}). Next, the server computes the cosine similarity with feature vector sub-parts using the HE-Cossim Ciphertext Cloning (HE-$C^3$) algorithm (Algorithm~\ref{alg:he_based_cosine_similarity_with_ccm}). Finally, the server combines the output ciphertexts into a single ciphertext using compression methods~\cite{choi2023blindtouch} and returns it to the client.

Our cosine similarity computation method divides the input ciphertext into $N_{in}$ expanded ciphertexts (Algorithm~\ref{alg:input_ctxt_expansion_methods}) and calculates the cosine similarity using the HE-$C^3$ algorithm (Algorithm~\ref{alg:he_based_cosine_similarity_with_ccm}). Our algorithms are designed to keep the multiplication ($Mul$) operations constant while reducing rotation ($Rot$) operations. As $N_{in}$ increases, the direct computation time for cosine similarity decreases. However, the time to generate $N_{in}$ ciphertexts increases. We have derived and empirically calculated the equation to determine the optimal $N_{in}$ that minimizes the total matching time.

\subsubsection{Input Ciphertext Expansion Algorithm}
\label{subsubsec:input_ctxt_expansion_algorithm}

Since generating $N_{in}$ ciphertexts on the client side is computationally expensive and incurs high network costs for transmitting them to the server, we propose an efficient method for expanding the client's single input ciphertext $C_u$ into $N_{in}$ ciphertexts on the server side, as described in Algorithm~\ref{alg:input_ctxt_expansion_methods}. This approach reduces the computational burden on the client and minimizes the amount of data transmitted over the network. 

The $\bold{mask}$ is redefined as a set of real vectors $\bold{mask}_{i}$, where $i \in [0, N_{in})$ with size $N$, in the following way: for $j \in [0, N)$, $\bold{mask}_{i}[i \cdot m/N_{in} + j] = 1$ if $(i \cdot m/N_{in} + j$ mod $m) < m/N_{in}$; otherwise, $\bold{mask_{i}[j]} = 0$.

\begin{algorithm}[!ht]
\caption{\textit{Input Ciphertext Expansion}}
\label{alg:input_ctxt_expansion_methods}
\begin{algorithmic}[1]
\STATE {\bfseries Input:} $C_{in}$ (= $C_u$) is an input ciphertext, $N_{in}$, $m$, $mask = \{mask_{i}\}_{i=1,2,...,N_{in}}$
\STATE {\bfseries Output:} $C_{out} = \{C_{out,i}\}_{i=1,2,...,N_{in}}$
\STATE Let $PoT(x) = 2^{x}$.
\FOR{$i=1$ {\bfseries to} $N_{in}$}
    \STATE $C_{out, i} = Res(Mul(P)(C_{in}, mask_{i}))$
    \FOR{$j=1$ {\bfseries to} $\log{N_{in}}$}
        \IF {$(i/PoT(j))$ mod $2 == 0$}
            \STATE $C_{rot, i} = Rot(C_{out, i}, (-1) \cdot PoT(j + \log{m}-\log{N_{in}}))$
        \ELSE
            \STATE $C_{rot, i} = Rot(C_{out, i}, PoT(j + \log{m} - \log{N_{in}}))$
        \ENDIF
        \STATE $C_{out,i} = Add(C_{out, i}, C_{rot, i})$
    \ENDFOR
\ENDFOR
\STATE \textbf{return} $C_{out}$
\end{algorithmic}
\end{algorithm}

\subsubsection{HE-$C^3$ Algorithm}
\label{subsubsec:HE_CCC_Algorithm}

The core idea of HE-$C^3$ is to use $N_{in}$ input ciphertexts instead of a single input ciphertext, as shown in Figure~\ref{fig:input_ctxt_for_CCM}.

\begin{figure}[!ht]
\centerline{\includegraphics[width=0.92\columnwidth]{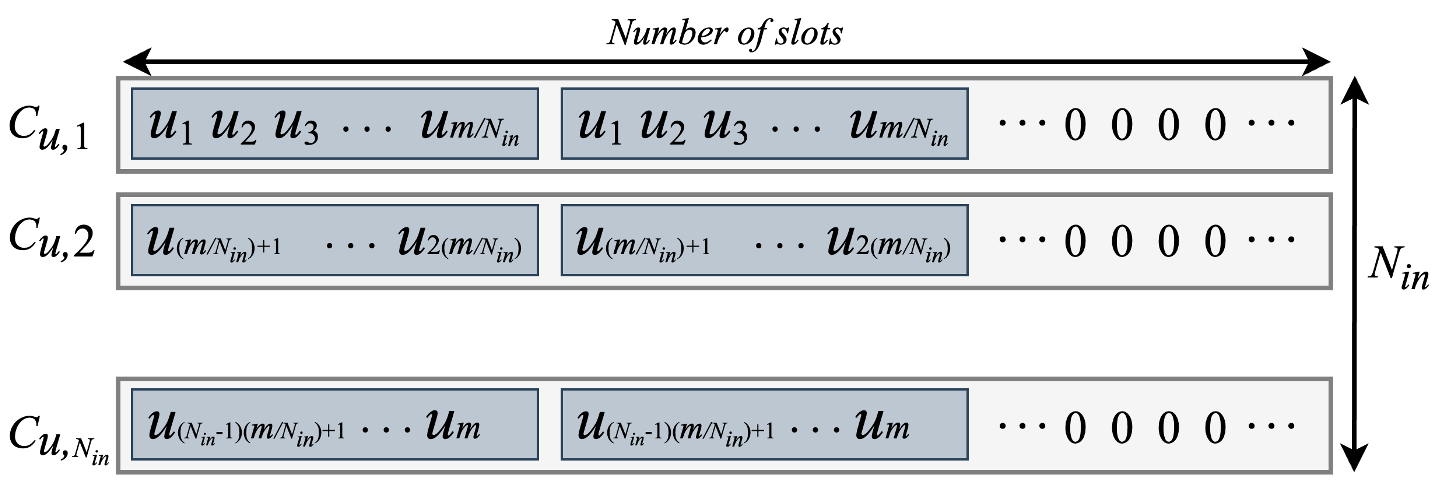}}
\caption{Input ciphertexts for HE-$C^3$.}
\label{fig:input_ctxt_for_CCM}
\end{figure}

The server then executes the matching algorithm described in Algorithm~\ref{alg:he_based_cosine_similarity_with_ccm}, which takes the following inputs: user's input ciphertexts $C_{u} = \{C_{u,i}\}_{i=1,2,...,N_{in}}$ from Algorithm~\ref{alg:input_ctxt_expansion_methods}, the server's stored ciphertexts $C_{s}=\{C_{s,i}\}_{i=1,2,...,N_{in}}$, and the feature vector size $m$. Algorithm~\ref{alg:he_based_cosine_similarity_with_ccm} initializes an output ciphertext $C_{out}$ with zero vectors. Then, for each $i$ from 1 to $N_{in}$, it multiplies the corresponding user and server ciphertexts, rescales the result, and adds it to $C_{out}$. Finally, it performs a series of rotations and additions on $C_{out}$ based on the values of $m$ and $N_{in}$, and returns the resulting ciphertext.

\begin{algorithm}[!ht]
\caption{\textit{HE-Cossim Ciphertext Cloning} (HE-$C^3$)}
\label{alg:he_based_cosine_similarity_with_ccm}
\begin{algorithmic}[1]
\STATE {\bfseries Input:} $C_{u} = \{C_{u,i}\}_{i=1,2,...,N_{in}}, C_{s}=\{C_{s,i}\}_{i=1,2,...,N_{in}}$, $m$
\STATE {\bfseries Output:} $C_{out}$
\STATE Let $C_{out}$ be a ciphertext of encryption of zero vectors.
\FOR{$i=1$ {\bfseries to} $N_{in}$}
\STATE $C_{out} = Add(C_{out}, Res(Mul(C_{u,i},C_{s,i})))$
\ENDFOR
\FOR{$i=1$ {\bfseries to} $(\log{m-\log{N_{in}}})$}
\STATE $C_{out} = Add(Rot(C_{out}, 2^{i-1}), C_{out})$
\ENDFOR
\STATE \textbf{return} $C_{out}$
\end{algorithmic}
\end{algorithm}


This approach significantly reduces the number of computationally expensive rotation operations in HE by effectively reducing the feature vector size from $m$ to $m/N_{in}$ by splitting the feature vectors into smaller parts. This directly reduces the number of $Add$ and $Rot$ operations required for the inner product operation in line 8 of Algorithm~\ref{alg:he_based_cosine_similarity_with_ccm}.

Although line 4 of Algorithm~\ref{alg:he_based_cosine_similarity_with_ccm} suggests an increase in $Add$, $Mul$, and $Res$ operations proportional to $N_{in}$, the decreased size of partial feature vectors ($C_{u, i}$ and $C_{s, i}$) by a factor of $N_{in}$ allows the ciphertext to contain $N_{in}$ times more partial feature vectors. As the ciphertext size remains constant, the total number of $Add$, $Mul$, and $Res$ operations does not increase with $N_{in}$ from an amortized analysis perspective, effectively reducing the number of $Add$ and $Rot$ operations without incurring additional computational costs.


\subsubsection{Compression Method}
\label{subsubsec:compression methods}
To minimize client-server communication costs, we employ a \textit{compression method} that combines multiple output ciphertexts into a single ciphertext. As shown in Figure~\ref{fig:look_and_feel}, the server applies this method to the expanded ciphertexts after executing Algorithm~\ref{alg:he_based_cosine_similarity_with_ccm}. 

\subsubsection{Decryption}
\label{subsec:decryption}
The server transmits the compressed result ciphertext $C_r$ to the client. The client decrypts $C_r$ using its secret key and identifies the index of the value with the highest similarity in the resulting vector. If this similarity value falls below a predetermined threshold, the client can conclude that user $u$ is not registered with the server.

\subsection{Cluster Architecture}
\label{subsec:cluster_architecture}
\sysname uses a scalable cluster architecture with a main server and multiple cluster servers for parallel processing of encrypted ciphertexts. This enables concurrent processing of $(N \cdot N_{in}) / m$ feature vectors. For example, with $N=8,192$, $m=128$, and $N_{in}=4$, 256 feature vectors are matched simultaneously, scaling overall matching time by $\lceil R / 256 \rceil$ ($R$ being total registered vectors).

During enrollment, encrypted user feature vectors are distributed across expandable clusters. For recognition, the client sends an input ciphertext to the main server, which distributes it to all clusters. Each cluster expands the input and executes Algorithm~\ref{alg:he_based_cosine_similarity_with_ccm} in parallel. The main server combines the returned outputs using the compression method before sending the final result to the client.

\section{Algorithm Analysis}
\label{sec:algorithm_analysis}

In this section, we find the optimal $N_{in}$ for minimizing the total matching time of \sysname. The time taken by the Input Ciphertext Expansion algorithm (Algorithm~\ref{alg:input_ctxt_expansion_methods}), HE-$C^3$ (Algorithm~\ref{alg:he_based_cosine_similarity_with_ccm}), and compression method are described in Propositions~\ref{prop:3}, \ref{prop:2}, and \ref{prop:4}, respectively. Using these propositions, we will demonstrate that the total matching time of \sysname is minimized when $N_{in}=4$ with $m = 128$, where $T_{OP,l}$ denotes the time to perform an OP (e.g., $Add$, $Mul$, $Rot$, and $Res$) operation at level $l$.

\begin{proposition}
\label{prop:3} 
The time it takes to perform Algorithm~\ref{alg:input_ctxt_expansion_methods} can be calculated as $N_{in} \cdot ((T_{Mul(P),l} + T_{Res,l}) + \log(N_{in}) \cdot (T_{Rot, l-1} + T_{Add, l-1}))$.
\end{proposition}
 
\begin{proposition}
\label{prop:2}
The time it takes to perform Algorithm~\ref{alg:he_based_cosine_similarity_with_ccm} for $R$ feature vectors is $\lceil (m^{\prime} / N_{in}) \rceil \cdot (N_{in} \cdot (T_{Mul(C), l} + T_{Res,l} + T_{Add,l-1}) + (\log{m}-\log{N_{in}}) \cdot (T_{Rot, l-1} + T_{Add,l-1}))$, where $m^{\prime} = (m\cdot R)/N$.
\end{proposition}

\begin{proposition}
\label{prop:4}
The time it takes to perform the \textit{Compression method} is $\lceil (m^{\prime}/N_{in}) \rceil \cdot (T_{Mul,l} + T_{Rot,l} + T_{Add,l})$, where $m^{\prime} = (m\cdot R)/N$.
\end{proposition}

The running time of Algorithm~\ref{alg:input_ctxt_expansion_methods} increases with $N_{in}$, while the running time of Algorithm~\ref{alg:he_based_cosine_similarity_with_ccm} decreases according to the inverse of $N_{in}$. Thus, the total matching time is minimized when $N_{in}$ is close to $\sqrt{m^{\prime}}$. To determine the optimal value of $N_{in}$, we use the actual execution times of each operation in the formulas. Table~\ref{table:operation_time_with_l} presents the average execution time ($ms$) for the CKKS scheme's operations ($Add$, $Mul(C)$, $Mul(P)$, $Rot$, and $Res$).

\begin{table}[t]
\caption{Comparison of the average execution time (with standard deviation) in milliseconds ($ms$) for the CKKS scheme's operations ($Add$, $Mul(C)$, $Mul(P)$, $Rot$, and $Res$) with $N=8,192$. All experiments were conducted 30 times.}
\centering
\resizebox{0.95\linewidth}{!}{
\renewcommand{\arraystretch}{1.3}
\begin{tabular}{cccccc}
\toprule[1.3pt]
$l$ \textbackslash \textit{op} & $Add$ & $Mul(C)$ & $Mul(P)$ & $Rot$ & $Res$ \\
\midrule
$3$ & 0.44 (0.03) & 7.87 (0.51) & 1.25 (0.13) & 7.50 (0.33) & 1.30 (0.04) \\
$2$  & 0.33 (0.01)  & 5.49 (0.34) & 0.97 (0.06) & 5.33 (0.45) & 0.95 (0.02) \\
$1$ & 0.09 (0.01) & 3.13 (0.10) & 0.76 (0.08) & 2.95 (0.16) & 0.60 (0.03) \\
\bottomrule[1.3pt]
\end{tabular}}
\label{table:operation_time_with_l}
\end{table}

By combining Propositions~\ref{prop:3}, \ref{prop:2}, and \ref{prop:4}, and using the values from Table~\ref{table:operation_time_with_l}, the total matching time $F(N_{in})$ as a function of $N_{in}$ can be written as:

{\footnotesize
\begin{align*}
F(N_{in})  
  &= \underset{\text{by Proposition~\ref{prop:3}}}{\underline{N_{in} \cdot ((T_{Mul(P),l} + T_{Res,l}) + \log(N_{in}) \cdot (T_{Rot, l-1} + T_{Add, l-1}))}} \\
  &+ \underset{\text{by Proposition~\ref{prop:2}}}{\underline{\left\lceil \left(\frac{m^{\prime}}{N_{in}}\right) \right\rceil \cdot  (N_{in} \cdot(T_{Mul(C), l-1}+T_{Res,l-1} + T_{Add, l-2})}}\\
  &+ \underset{\text{by Proposition~\ref{prop:2}}}{\underline{(\log{m}-\log{N_{in}})\cdot(T_{Rot, l-2}+T_{Add, l-2})}})\\
  &+ \underset{\text{by Proposition~\ref{prop:4}}}{\underline{\left\lceil \left(\frac{m^{\prime}}{N_{in}}\right) \right\rceil \cdot (T_{Mul(P),l-2}+T_{Rot, l-2}+T_{Add, l-2})}}
\end{align*}
}

Evaluating the formula for different values of $N_{in}$ with Table~\ref{table:operation_time_with_l}, we get: $F(2)=578.02, F(4)=416.44, F(8)=429.04, F(16)=637.84$, and $F(32)=1206.04$. Therefore, the minimum matching time is obtained when $N_{in}=4$, with a value of $F(N_{in}) = 416.44$.
\section{Experiments}
\label{sec:experiments}

\subsection{Experimental Setup}
\label{subsec:experimental_setup}

We utilized Lattigo v5~\cite{lattigo}, an open-source library for HE, with $N=8,192$ and $\log{PQ} \approx 158$ to ensure a 128-bit security level~\cite{rahulamathavan2022privacypreserving}. Our architecture consisted of one client, one server, and three clusters (see Section~\ref{subsec:cluster_architecture}). The experiments were conducted using four NAVER Cloud~\cite{navercloudServer} standard-g2 server instances, each equipped with two cores (Intel(R) Xeon(R) Gold 5220 CPU @ 2.20GHz) and 8GB of DRAM. Each cluster was designed to support 2,048 biometric identifications, resulting in a total capacity of simultaneously storing 6,144 biometric data entries across the three clusters. We employed ResNet-18 (denoted as R18)~\cite{he2015deep} as the feature extractor for both fingerprints and faces, and the model was trained using the ArcFace loss function to enhance its discriminative power. Since R18 is a highly light CNN architecture, we utilize R18 instead of R50 or larger CNN architectures to reduce the load of the client's device. 

\subsection{Datasets}
\label{subsec:datasets}

We conducted extensive experiments using five face and three fingerprint datasets to evaluate the performance of \sysname.

\subsubsection{Face Datasets}
\label{subsubsec:faced}

For 1:N face matching training, we primarily used the Glint360k dataset~\cite{glint360_pfc}, which contains 93K identities and 5.1M images. We followed the preprocessing steps as described in ArcFace~\cite{arcface2019} to ensure consistency with state-of-the-art methods. To evaluate our model's accuracy, we used several benchmark datasets, including LFW~\cite{LFWTech}, CFP-FP~\cite{cfpfp}, AgeDB~\cite{agedb}, and IJB-C~\cite{ijbc}. These datasets cover a wide range of variations in facial appearances, poses, and ages, allowing for a comprehensive assessment of our model's performance in real-world scenarios.

\subsubsection{Fingerprint Datasets}
\label{subsubsec:fingerprint_datasets}

We evaluated our model's 1:N fingerprint matching performance using four established datasets: FVC2002 and FVC2004 \cite{maio2002fvc2002}, PolyU Cross Sensor \cite{lin2018matching}, and CASIA Version 5.0 \cite{CASIA}. PolyU includes both contact-based and contactless-2D images, while CASIA, the largest public dataset, contains 20,000 images from 4,000 subjects. To improve generalizability, we combined FVC subsets and enhanced image quality in FVC and PolyU through segmentation and centralization.

\subsection{Execution Time of \sysname}
\label{subsec:exercution_time}

We evaluated how efficiently \sysname can process 6,144 IJB-C samples when HE operations are applied. Table~\ref{table:experim_results_about_T_128} presents the execution time for each step in \sysname with varying $N_{in}$ values.

\begin{table}[!ht]
\caption{Mean 1:N matching time ($ms$) over 10 trials for \sysname with different $N_{in}$ on 6,144 IJB-C samples with the following parameters: $N=8,192$, $l=3$, feature vector size = $128$. Values in parentheses represent standard deviations.}
\centering\small
\centering\resizebox{\linewidth}{!}{
\renewcommand{\arraystretch}{1.3}
\begin{tabular}{c|c c c c}
\noalign{\smallskip}\noalign{\smallskip}
\toprule[1.3 pt]
\textbf{Operation (OP)} & 2 & 4 & 8 & 16 \\
\hline
\hline
Encryption & \multicolumn{4}{c}{29.58 (4.32)} \\
\hline
Decryption & \multicolumn{4}{c}{24.35 (1.57)} \\
\hline
Inference & \multicolumn{4}{c}{129.23 (8.24)}  \\
\hline
Matching & 650.92 (13.38) & 451.67 (9.27)& 457.02 (13.92) & 652.13 (13.45) \\
Network & 108.52 (7.39) & 102.60 (12.71) & 100.03 (12.50) & 105.55 (19.93) \\ \hline \hline
\textbf{Total} & 942.59 (18.59) & 737.41 (13.49) & 740.20 (12.72) & 940.84 (12.73) \\
\hline
\toprule[1.3 pt]
\end{tabular}}
\label{table:experim_results_about_T_128}
\end{table}

Table~\ref{table:experim_results_about_T_128} presents the 1:N matching time using a ResNet-18-based CNN extractor with a 128-dimensional feature vector. The network time measurements confirm that $N_{in}$ minimally impacts \sysname's network time. As the CKKS parameter setting remains constant, encryption and decryption times are identical across all $N_{in}$ values. The results show optimal matching time (451.67 ms) and total time (737.41 ms) when $N_{in} = 4$, aligning with the analysis in Section~\ref{sec:algorithm_analysis}. \sysname's network time remains consistent at around 110 ms for 6,144 biometric 1:N matchings, outperforming Blind-Touch's 139 ms for 5,000 fingerprint matchings. 


\noindent\textbf{Comparison to Conventional HE-based 1:N Biometric Matching Algorithm.}
We compared the execution time of \sysname's 1:N matching algorithm with the conventional cosine similarity-based 1:N matching algorithm (\textit{Base}). Table~\ref{table:comparisoin_of_he_cossim} shows the matching time comparison, revealing that \sysname's 1:N matching time is significantly reduced compared to the conventional approach. For a feature vector size of 128, \sysname achieves a matching time of 451.67 ms, which is nearly 3.5 times faster than \textit{Base}. Furthermore, \sysname demonstrates its efficiency across different feature vector sizes, with matching times of 136.51 ms and 784.91 ms for feature vector sizes of 16 and 256, respectively.

\begin{table}[!ht]
\caption{Comparison of mean 1:N matching time ($ms$) over 10 trials comparison between the conventional cosine similarity-based algorithm (\textit{Base}) with a feature vector size of 128 and \sysname with feature vector sizes of 16, 128, and 256, evaluated on 6,144 samples. Values in parentheses represent the standard deviations.}
\centering\resizebox{\linewidth}{!}{
\renewcommand{\arraystretch}{1.3}
\begin{tabular}{c | c c c c}
\noalign{\smallskip}\noalign{\smallskip}
\toprule[1.3 pt]
Operation & \textit{Base} (128) & \sysname (16) & \sysname (128) & \sysname (256)\\
\hline
\hline
Matching & 1,577.90 (11.43) & 136.51 (4.55) & 451.67 (9.27) & 784.91 (11.35)\\
\hline
\toprule[1.3 pt]
\end{tabular}}
\label{table:comparisoin_of_he_cossim}
\end{table}



\subsection{Accuracy of \sysname}
\label{subsec:performance}
We compared the accuracy of \sysname with state-of-the-art methods under various conditions.

\noindent\textbf{Accuracy of 1:N Face Matching.} Table~\ref{table:face_performance_table} presents the accuracy of \sysname with varying feature vector sizes, compared to the state-of-the-art model CosFace with a ResNet-50 (R50) backbone and feature size 512~\cite{cosface2018}. For CosFace, we used the results reported in the original paper, which did not include IJB-C. CosFace with R50 and feature size 512 outperforms \sysname using R18 with smaller feature sizes. While the performance difference is not significant on LFW, \sysname shows slightly lower accuracy on CFP-FP and AgeDB, even with a feature size of 512. The impact of feature vector size is minimal on LFW, but performance degrades on CFP-FP and AgeDB with feature sizes of 64 or less. For IJB-C, accuracy decreases with feature sizes of 128 or less. On AgeDB and IJB-C, feature vector size significantly impacts accuracy, with a size of 16 resulting in a substantial performance drop. Considering the matching time comparison in Table~\ref{table:comparisoin_of_he_cossim}, we recommend a feature vector size of 128 for \sysname to minimize the negative impact on accuracy and throughput while maintaining competitive performance.

\begin{table}[!t]
\caption{Accuracy of \sysname on face recognition benchmarks for different feature vector sizes, compared to the state-of-the-art CosFace model with ResNet-50 (R50) and feature size 512. Rank-1 accuracy is reported for all datasets except IJB-C, which was not reported~\cite{cosface2018}. Values in parentheses represent the feature vector size used in each model.}
\centering\resizebox{0.86\linewidth}{!}{
\renewcommand{\arraystretch}{1.3}
\begin{tabular}{@{}c|cccc@{}}
\noalign{\smallskip}\noalign{\smallskip}
\toprule[1.3 pt]
Method & LFW & CFP-FP & AgeDB & IJB-C \\
\hline
\makecell{CosFace (512), R50~\cite{cosface2018}} & 99.83 & 99.33 & 98.55 & - \\
\sysname (512) & 99.72 & 95.61 & 97.28 & 95.32 \\
\sysname (256) & 99.58 & 94.91 & 97.43 & 95.36 \\
\sysname (128) & 99.63 & 95.54 & 97.18 & 94.91 \\
\sysname (64) & 99.52 & 94.81 & 96.10 & 93.48 \\
\sysname (32) & 99.43 & 94.36 & 95.10 & 90.86 \\
\sysname (16) & 99.22 & 94.00 & 90.78 & 83.83 \\ 
\toprule[1.3 pt]
\end{tabular}}
\label{table:face_performance_table}
\end{table}

\noindent\textbf{Accuracy of 1:N Fingerprint Matching.}
We analyzed \sysname's performance on the PolyU contactless fingerprint dataset with varying feature vector sizes. The PolyU dataset consists of two sessions, each containing 336 and 160 subjects, respectively, with 6 fingerprint images per subject. The first session was used for training and the second for testing (see Table~\ref{table:polyU_feature_vector_test}).

\begin{table}[!ht]
\small
\caption{Rank-1 accuracy of \sysname on the PolyU Contactless Fingerprint dataset for different feature sizes.}
\centering\resizebox{0.79\linewidth}{!}{
\renewcommand{\arraystretch}{1.3}
\begin{tabular}{c | c c c c c c}
\noalign{\smallskip}\noalign{\smallskip}
\toprule[1.3 pt]
Feature size & 512 & 256 & 128 & 64 & 32 & 16\\
\hline
\hline
Rank-1 (\%) & 99.79 & 99.70 & 99.68 & 99.64 & 99.64 & 99.55 \\
\hline
\toprule[1.3 pt]
\end{tabular}}
\label{table:polyU_feature_vector_test}
\end{table}

\begin{table}[!ht]
\small
\caption{Rank-1 accuracy of \sysname on the FVC Fingerprint dataset for different feature sizes.}
\centering\resizebox{0.79\linewidth}{!}{
\renewcommand{\arraystretch}{1.3}
\begin{tabular}{c | c c c c c c}
\noalign{\smallskip}\noalign{\smallskip}
\toprule[1.3 pt]
Feature size & 512 & 256 & 128 & 64 & 32 & 16\\
\hline
\hline
Rank-1 (\%) & 90.52 & 90.47 & 90.49 & 90.43 & 90.30 & 90.18 \\
\hline
\toprule[1.3 pt]
\end{tabular}}
\label{table:fvc_feature_vector_test}
\end{table}

\sysname demonstrates superior performance across multiple datasets and feature sizes. On the PolyU dataset, it achieves 99.79\% Rank-1 accuracy with a 512-feature size and 99.68\% with 128 features, significantly outperforming Blind-Touch~\cite{choi2023blindtouch}, which achieves only 59.17\%. Notably, \sysname's performance remains robust even with a 16-feature size, showing only a 0.24\% decrease in accuracy compared to the 512-feature model. On the FVC dataset, \sysname maintains high performance with Rank-1 accuracies of 90.52\%, 90.49\%, and 90.18\% for feature sizes 512, 128, and 16, respectively. The slight performance drop (0.34\%) from 512 to 16 features is negligible. The lower overall accuracy on FVC compared to PolyU can be attributed to higher noise levels and the use of multiple fingerprint readers. For the CASIA dataset, \sysname achieves impressive Rank-1 accuracies of 99.97\% and 99.87\% with feature sizes 16 and 128, respectively.

Given the consistent high performance across feature sizes and the small-resolution, black-and-white nature of the PolyU, FVC, and CASIA datasets, we recommend using a 16-feature size for \sysname in 1:N fingerprint matching tasks. This configuration offers an optimal balance between computational efficiency and accuracy. Table~\ref{table:fvc_feature_vector_test} provides detailed results for the FVC dataset.

\noindent\textbf{Accuracy of 1:1 Fingerprint Matching.}
We compared AUC and EER scores for the PolyU dataset with state-of-the-art 1:1 fingerprint-matching architectures. The dataset was split into train and test sets using the configuration in~\cite{feng2023detecting}, with 3,000 genuine pairs and 19,900 imposter pairs in the test set. All models except \sysname (feature size 128) and Blind-Touch~\cite{choi2023blindtouch} used plaintext data. Table~\ref{table:score_of_polyU_dataset} shows that \sysname's 1:1 matching accuracy, in terms of AUC and EER scores, is close to that of state-of-the-art architectures~\cite{maltoni2022handbook}. When trained on 1:N matching, \sysname achieves an AUC score of 98.55\%, only 0.78\% lower than the best plaintext model, ContactlessMinuNet \cite{zhang2021multi}. However, the EER score of 5.9\% is slightly higher due to \sysname's optimization for 1:N matching. This optimization prioritizes overall performance across thresholds rather than optimizing for a specific threshold, which explains the high AUC alongside the higher EER.

\begin{table}[t]
\small
\caption{AUC and EER of \sysname and state-of-the-art methods on the PolyU dataset for 1:1 fingerprint matching.}
\centering\resizebox{0.7\linewidth}{!}{
\renewcommand{\arraystretch}{1.3}
\begin{tabular}{c| c c}
\noalign{\smallskip}\noalign{\smallskip}
\toprule[1.3 pt]
Method & AUC (\%) & EER (\%) \\
\hline\hline
MNIST mindtct~\cite{ko2007user} & 58.91 & 36.85 \\
MinutiaeNet \cite{nguyen2018robust} & 93.03 & 13.35 \\
VeriFinger (paid software) & 98.16 & 2.99 \\
ContactlessMinuNet \cite{zhang2021multi} & 99.33 & 1.94 \\
MinNet~\cite{feng2023detecting} & 99.25 & 1.90 \\
Blind-Touch~\cite{choi2023blindtouch} & 97.50 & 2.50 \\
\sysname (128) & 98.55 & 5.90 \\
\hline
\toprule[1.3 pt]
\end{tabular}}
\label{table:score_of_polyU_dataset}
\end{table}

\section{Security Analysis}
\label{sec:security_analysis}

\sysname considers an honest-but-curious server adversary model, where the server follows the protocol but can observe encrypted feature vector ciphertexts. The system's security relies on the CKKS scheme~\cite{cheon2017homomorphic}, a fully HE scheme that provides semantic security against chosen-plaintext attacks (IND-CPA). The security of the CKKS scheme is based on the hardness of the decisional ring learning with errors problem~\cite{lyubashevsky2013ideal}, ensuring that an adversary cannot distinguish between encryptions of different messages.

In \sysname, the client encrypts feature vectors using the CKKS scheme before sending them to the server, which performs matching on the encrypted vectors without decrypting them. Due to the IND-CPA security of the CKKS scheme, the server learns nothing about the plaintext feature vectors from the ciphertexts except for their length. The feature extractor always produces ciphertexts of consistent length, ensuring no additional information can be inferred. More formally, the security of \sysname can be analyzed using the simulation paradigm~\cite{lindell2017simulate}, where an ideal functionality receives feature vectors, performs matching, and returns the result to the client. Any attack on \sysname can be simulated as an attack on this ideal functionality, implying the system's security against honest-but-curious adversaries. However, HE only provides confidentiality and does not protect against integrity issues or denial-of-service attacks.

\section{Deployment Plan and Conclusion}
\label{sec:deploy_and_conclusion}

Our proof-of-concept implementation demonstrated \sysname's exceptional efficiency, matching over 5,000 face images and performing 6,144 identifications in just 0.74 seconds across multiple biometric modalities. This establishes \sysname as the first real-time homomorphic encryption-based architecture for 1:N biometric matching with practical applicability.

Future work will deploy \sysname in real-world scenarios, expanding its functionality to Android devices via a cloud-based matching server. To evaluate performance and scalability, we plan to implement \sysname in a large research institution's entrance system, serving over 4,000 employees.

\section*{Acknowledgments}
This work was partly supported by grants from the IITP (RS-2022-II220688, RS-2021-II212068, No.2022-0-01199, RS-2023-00229400, RS-2019-II190421, and RS-2024-00439762).

\bibliography{cikm}

\end{document}